\documentclass{IEEEtran}
\usepackage{amsmath,amsfonts,amsthm}
\usepackage{array}
\usepackage{cite}
\usepackage{stfloats}
\newcolumntype{?}{!{\vrule width 1.5pt}}
\usepackage{bm}

\newcommand{\citet}[1]{\cite{#1}}
\newcommand{\citep}[1]{\cite{#1}}
\newcommand{\argmax}{\operatornamewithlimits{argmax}}
\newcommand{\argmin}{\operatornamewithlimits{argmin}}
\newcommand{\defeq}{\stackrel{\text{def}}{=}}

\newtheorem{thm}{Theorem}
\newtheorem{lemma}{Lemma}
\newtheorem{cor}{Corollary}[thm]

\newtheorem{defi}{Definition}

\begin{document}
\title{MML is not Consistent for Neyman-Scott}
\author{Michael~Brand$^*$%
\thanks{M.~Brand is with Otzma Analytics Pty Ltd and with the Faculty of IT, Monash University, Clayton, VIC 3800, Australia. e-mail: research@OtzmaAnalytics.com}}

\maketitle

\begin{abstract}
Strict Minimum Message Length (SMML) is an information-theoretic statistical inference method widely cited (but only with informal arguments) as providing estimations that are consistent for general estimation problems. It is, however, almost invariably intractable to compute, for which reason only approximations of it (known as MML algorithms) are ever used in practice.
Using novel techniques that allow for the first time direct, non-approximated
analysis of SMML solutions,
we investigate the Neyman-Scott estimation problem, an oft-cited showcase for the consistency of MML, and show that even with a natural choice of prior neither SMML nor its popular approximations are consistent for it, thereby providing a counterexample to the general claim. This is the first known explicit construction of an SMML solution for a natural, high-dimensional problem.
\end{abstract}


\begin{center} \bfseries EDICS Category: MAL-d \end{center}

\IEEEpeerreviewmaketitle

\section{Introduction}

\IEEEPARstart{I}{n} the context of statistical inference and point estimation,
the term \emph{consistency} refers to the ability of an estimator to converge
with probability $1$ to the correct value of the parameter it is estimating
(whatever that parameter's true value may be)
as the number of observations grows to infinity.
(See \citet{lehmann2006theory} for a formal definition.)
Typically, this property is discussed in the context of a specific estimation
problem or when estimating a specific statistic
(e.g., in determining whether or not the average of
independent observations is a consistent estimator for the expectation of
the distribution generating them).

An estimator is said to be consistent (without specifying the estimation
problem) when its consistency property is universal, i.e.\ holds for any choice
of estimation problem. Most popular estimators are not consistent in this
universal sense. For example, \citet{lehmann2006theory} lists explicit
conditions for Maximum Likelihood (ML) estimation to be consistent, and
provides examples of estimation problems for which it is not.

One estimator long believed to be (universally) consistent, however,
\citep{DoweBaxterOliverWallace1998,dowe1997resolving,DoweGardnerOppy2007,BarronCover1991}
is \emph{Minimum Message Length} (MML), and particularly the
Strict MML (SMML) estimator. In \citet{dowe2011mml}, this property is even
considered one of MML's defining characteristics.

MML is a general name for any member
of the family of Bayesian statistical inference methods based on the
information-theoretic minimum message length principle,
which, in turn, is closely related to the family of Minimum Description Length
(MDL) estimators
\citep{rissanen1985minimum,rissanen1987stochastic,rissanen1999hypothesis},
but predates it.
The minimum message length principle
was first introduced in \citet{WallaceBoulton1968}, and the
estimator that follows the principle directly, which was first described in
\citet{wallace1975invariant}, is known as Strict MML (SMML).
We describe it formally in Section~\ref{SS:smml}.

One estimation problem of particular interest in the context of consistency
is the Neyman-Scott problem \citep{neyman1948consistent}.
This is defined as follows.

\begin{defi}
Let $\mu$ be the vector $(\mu_1,\ldots,\mu_N)$.

The \emph{Neyman-Scott problem}
is the problem of jointly estimating the tuple
$(\sigma^2,\mu)$ after observing
$(x_{nj} : n=1,\ldots,N; j=1,\ldots,J)$, each element of which is independently
distributed $x_{nj} \sim N(\mu_n,\sigma^2)$, where $N(\mu_n,\sigma^2)$ is the
normal distribution with mean $\mu_n$ and variance $\sigma^2$.

It is assumed that $J\ge 2$.
\end{defi}

If we let
\[
m_n \defeq \frac{\sum_{j=1}^{J} x_{nj}}{J}
\]
and
\[
s^2 \defeq \frac{\sum_{n=1}^{N} \sum_{j=1}^{J} (x_{nj}-m_n)^2}{NJ},
\]
and if we also let $m$ be the vector $(m_1,\ldots,m_N)$, then $(s^2,m)$ is a
sufficient statistic for this problem, which is why the observables used for
estimating $(\sigma^2,\mu)$ can be taken to be $(s^2,m)$, rather than the
$x$ values directly.

The interesting
case for Neyman-Scott is to observe the behaviour of the estimate for
$\sigma^2$ when this estimate is part of the larger joint estimation problem,
while taking $N$ to infinity and fixing $J$.

Importantly, this set-up is beyond the standard asymptotic regime for
consistency. In standard consistency, one would require the estimator to
estimate $\sigma^2$, and its consistency in doing so would be determined
based on whether this estimate converges to any true value of $\sigma^2$ with
probability $1$. In the Neyman-Scott set-up, the estimator is required to
estimate all of $(\sigma^2,\mu)$, but then only the convergence properties
of the first element are investigated.

This is ostensibly a harder case than standard consistency, because it creates
an \emph{inconsistent posterior} \citep{ghosal1997review},
a situation where even with unlimited data, the uncertainty regarding the
true value of $(\sigma^2,\mu)$ remains high, even though the value of $\sigma^2$
is known with high confidence.

If we describe a (prior) distribution of $\mu$ given $\sigma^2$, and thus
likelihood functions of the observed variables $x$ as a function of $\sigma^2$,
rather than of the tuple $(\sigma^2,\mu)$, the problem of estimating
$\sigma^2$ (alone) is
a standard consistency problem, and for this alternate set-up standard
estimators such as ML are, indeed, consistent. In the special form of
consistency required for standard Neyman-Scott, however (which is sometimes
referred to as ``internal consistency''),
many of the popular estimation methods fail to return a consistent estimate
for $\sigma^2$. Maximum Likelihood, as a case in point, returns
the (inconsistent) estimate $s^2$, rather than $\frac{J}{J-1} s^2$.

MML, on the other hand, has long been believed to be consistent even for the
Neyman-Scott problem
\citep{makalic2012mml,makalic2009minimum,yatracos2015mle,laskar2001modified},
and, more generally, for the wider class of problems of a ``Neyman-Scott
nature'' \citep{dowe1997resolving}, making the Neyman-Scott problem a powerful
and oft-cited showcase for MML's superior consistency properties.

The reasons for the belief in MML's consistency, both in the standard
scenario and in the Neyman-Scott one, are due to MML's theoretical underpinnings
\citep{dowe1997resolving,DoweGardnerOppy2007}, which
we describe in Section~\ref{SS:smml}. Importantly, MML has never been formally
proved to exhibit universal consistency in either scenario.
In fact, there has not been any case for which the SMML estimate is
known, is known to be consistent, and the ML estimate is not.

This lack of ``empirical'' evidence is due to the fact that SMML is
computationally and analytically intractable in all but a select
few single-parameter cases \citep{dowty2015smml}.
It was proved to be NP-hard to compute in general
\citep{farrwallace2002}.

SMML is for this reason only ever used on natural problems by means of one
of its computationally-feasible approximations, known as \emph{MML algorithms}, 
the most widely used of which perhaps being the Wallace-Freeman approximation
(WF-MML) \citet{WallaceFreeman1987}. In this approximated form, while still
not as popular as ML or Maximum A Posteriori (MAP), MML enjoys a wide following,
with over 90 papers published regarding it in 2018 alone, including
\citet{cheeseman2018bayesian,zamzami2018mml,li2018subspace,channoufi2018color}.
However, the strong consistency properties attributed to SMML are not
believed to be as universal for its approximations, for which reason they
cannot be used to provide counterexamples to the general claim.

Nevertheless, as a demonstration of MML's consistency,
\citet{dowe1997resolving} calculated the Wallace-Freeman MML estimate
for the Neyman-Scott problem, and \cite[Sections 4.2--4.9]{Wallace2005}
expanded on this by using another MML approximation, due to Dowe and Wallace,
known as
``Ideal Group'' (IG), which is in some ways a more direct approximation to
SMML but often as intractable as SMML itself.
In both cases, the estimates proved to be consistent.

In this paper, we develop novel methods that allow for the first time direct,
non-approximated analysis of a high-dimensional SMML solution.
Inevitably, given the hardness results regarding SMML, it is not possible
for such methods to be universally applicable to any estimation problem.
However, they can be used in a known, broad class of problems, that we
refer to as \emph{regular} problems. This class includes, among others,
natural problems, including
high-dimensional ones. In particular, we prove that the Neyman-Scott
problem is regular, and use our techniques to analyse the behaviour
of SMML on it, without any assumptions or approximations.

This is done, however, with one caveat.
The Neyman-Scott problem is a frequentist problem, defined by its
likelihoods. To put any Bayesian method, including any of the MML variants
discussed, to use on an estimation problem, one must also define a prior
distribution for the estimated parameters. Bayesian methods accept such
priors as part of the problem description, i.e.\ as given.
Without a specified prior, the
``Neyman-Scott problem'' is, from a Bayesian viewpoint, an entire family of
estimation problems.

The Neyman-Scott problem has been
investigated in the literature with many priors, proper and improper,
informative and uninformative.
The investigations of both  \citet{dowe1997resolving} and \citet{Wallace2005}
are of the Neyman-Scott problem over the prior function $1/\sigma$,
which we will refer to as the \emph{Wallace prior}.
This is a standard, improper, uninformative prior for the problem. However,
it is not the only such prior.
Four such priors that are in common use for the Neyman-Scott problem are
listed, for example, in \citet{yang1996catalog}. In this paper we analyse a
different one of the four, $1/\sigma^{N+1}$, which we refer to as the
\emph{scale free} prior.
Both are commonly-used priors that are in no way considered pathological
for the problem.\footnote{In terms of the properties exhibited by the two
priors discussed,
both have an improper scale-free (i.e., $1/\sigma$) distribution on $\sigma$ and
both have an improper uniform distribution on $\mu$ given $\sigma$, but in
the scale-free prior the $\mu_n$ are individually scale free (i.e., have a
$1/|\mu_n|$ distribution), whereas in the Wallace prior they are individually
uniformly distributed and are independent of $\sigma$.

It should be noted that the scale-free prior's dependence on $N$ (which,
at first sight, might seem suspect to readers more familiar with the
Wallace prior, especially in the context of consistency analysis) is
merely an artefact of the problem's parameterization.
If, instead of using the parameters $(\sigma,\mu)$,
we define the problem over the parameters $(\sigma,\zeta)$, where
$\mu=\zeta\sigma$,
the scale-free prior becomes $1/\sigma$,
identical to the Wallace prior in the original parameterization. Because both
SMML and all MML approximations discussed are invariant to
parameterization, all our conclusions regarding the consistency of these
algorithms are equally applicable for either parameterization.}

The reason we require this alternate prior is that it is the only prior for
which the Neyman-Scott problem is regular. We refer to the
Neyman-Scott problem under the scale-free prior as the
\emph{Scale-free Neyman-Scott problem}. (See Section~\ref{SS:types} for an
explanation of the name ``scale free''.)

By applying our method, we show that SMML is not consistent
for the scale-free Neyman-Scott problem, thus giving a
counterexample to the general claim regarding SMML's universal
consistency properties in Neyman-Scott-like scenarios.\footnote{The
paper's title, ``MML is not consistent for Neyman-Scott'',
should be interpreted as the logical opposite to this general claim,
i.e.\ ``It is not true
that MML is consistent for a general member of the Neyman-Scott estimation
problem family; it will be inconsistent for some Neyman-Scott cases.''}

More generally, our results serve as a strong indication that there is no
reason to
assume SMML holds consistency properties that are superior to those of ML,
because our methods relate SMML to ML not just for Neyman-Scott but also for the
general class of regular problems.
It is in this context that our finding that MML is, in fact,
no better than ML for Neyman-Scott becomes highly significant for
MML at large.

For completion, in Appendix~\ref{S:approximations}
we demonstrate that the MML approximations
used by \citet{dowe1997resolving} and \citet{Wallace2005} converge, for this
estimation problem, to the same limit as SMML, and are therefore also
not consistent for this problem.

\section{Definitions}\label{S:definitions}

\subsection{Notation}\label{SS:notations}

This paper deals with the problem of statistical inference: from a set of
observations, $x$, taken from $X$ (the observation space)
we wish to provide a point estimate,
$\hat{\theta}(x)$, to the value, $\theta$, of a random variable,
$\boldsymbol{\theta}$, drawn from $\Theta$ (parameter
space). When speaking about statistical inference in general, we
use the symbols introduced above. For a specific problem, such as in discussing
the Neyman-Scott problem, we use problem-specific names for the variables.
However, in all cases Latin characters refer to observables, Greek to
unobservables that are to be estimated, boldface characters to random variables,
non-boldface characters to values of said random variables, and hat-notation
to estimates. Boldface is used for the observations, too, when
considering the observations as random variables.

All point estimates discussed in this paper
are defined using an $\argmin$ or an $\argmax$. We take these functions
as, in general, returning sets. Nevertheless, we use
``$\hat{\theta}(x)=\theta$'',
as shorthand for
``$\theta\in\hat{\theta}(x)$'', because in typical usage the maxima/minima
are unique, and the sets returned are therefore singletons. This is what
makes the estimates discussed point estimates.

To be consistent with the notation of \citet{Wallace2005}, we use
$h(\theta)$ to indicate the prior and
\begin{equation}\label{Eq:r}
r(x)=\int_\Theta h(\theta)f(x|\theta)\text{d}\theta
\end{equation}
as the marginal.
The integral of $h(\theta)$ over $\Theta$ may be $1$ (in which case it is
a \emph{proper prior} and the problem is a \emph{proper estimation problem})
but it may also integrate to other positive values (in which case it is a
\emph{scaled prior}) or diverge to infinity (in which case it is an
\emph{improper prior}).
Our analysis will reject a prior as \emph{pathological} only
if it does not allow computation of a marginal using \eqref{Eq:r}.

When speaking of events that have positive probability, we will use the
$\text{Prob}()$ notation. However, in calculating over a scaled or
improper prior some probabilities will be correspondingly scaled when computed
as an integral over the prior or the marginal. For these
we use the $\text{ScaledProb}()$ notation.

For reasons of mathematical convenience, we take both the observation space,
$X$, and the parameter space, $\Theta$, as complete metric spaces, and
assume that priors, likelihoods, posterior probabilities and marginals
are all continuous, differentiable, everywhere-positive functions.
This allows us to take limits,
derivatives, $\argmin$s, $\argmax$s, etc., freely, without having to
prove at every step that these are well-defined and have a value.

\subsection{MML}\label{SS:smml}

Minimum Message Length (MML) \citep{Wallace2005}
is an inference method that attempts to codify in information-theoretic terms
the principle of Occam's Razor.

Consider $F:X\to\Theta$ as a candidate point estimation function. To
evaluate the suitability of $F$, consider
an observer wanting to communicate $x$. Such an observer may do so
using a two-part message,
first communicating $F(x)$ as an estimate to the value of $\theta$, 
and then communicating $x$ on the assumption that $\theta$ equals the
communicated $F(x)$. The first part of the message is, in expectation,
shorter the ``simpler'' $F$ is (in the sense of having lower entropy),
the second part is, in expectation, shorter the more representative $F$'s
estimations. The $F$
with the shortest total expected message length is therefore, under this model,
the one best fitting Occam's ideal of choosing the simplest hypothesis
that still adequately fits the available data.

When $X$ has only countably many elements, an $F$ can be chosen such that
both message parts are finite. In this case, $F$ necessarily maps to only
countably many distinct $\theta$ values.
We label these $\theta_1,\theta_2,\ldots$.

The expected length of the first part of the message is
\begin{align*}
L_E(F)&\defeq \sum_i -\text{Prob}(F(\mathbf{x})=\theta_i)\log \text{Prob}(F(\mathbf{x})=\theta_i) \\
&= H(F(\mathbf{x})),
\end{align*}
where $H$ is the Shannon entropy.\footnote{Here and everywhere in the paper,
all logs are natural logs, and information is measured in nits.}

Consider $\mathbf{i}$, the random variable such that
$F(\mathbf{x})=\theta_\mathbf{i}$.
The expected length of the second part of the message is the expectation,
over $i=\mathbf{i}$, of the message length of an $x$ value taken from the
distribution of $\mathbf{x}$
given $F(\mathbf{x})=\theta_i$, when $x$ is encoded optimally for its
distribution under $\boldsymbol{\theta}=\theta_i$, i.e.\ the cross entropy of
these two distributions. In a formula, this is
\begin{align}
\sum_i \text{Prob}(F(\mathbf{x})=\theta_i) \mathbf{E}_{\mathbf{x}|F(\mathbf{x})=\theta_i}[-\log \text{Prob}(\mathbf{x}|\boldsymbol{\theta}=\theta_i)] \nonumber\\
=\sum_{x\in X} \text{Prob}(\mathbf{x}=x)\log\left(\frac{1}{\text{Prob}(\mathbf{x}=x|\boldsymbol{\theta}=F(x))}\right),\label{Eq:finite}
\end{align}
where $\mathbf{E}$ signifies expectation.

Unfortunately, this method does not work when $X$ is uncountable, because then
the second part of the message becomes infinite and cannot be minimised. The
way to solve this problem is to subtract from \eqref{Eq:finite} the prior
entropy of $\mathbf{x}$, $H(\mathbf{x})$, which is clearly independent of the
choice of $F$. The remainder is known as the \emph{excess message
length}.
This is guaranteed to be finite, even when $X$ is uncountable.

Returning now to our notation for continuous variables, if we define
\begin{equation}\label{Eq:Rdef}
R_{\theta}(x)\defeq\log\left(\frac{r(x)}{f(x|\theta)}\right),
\end{equation}
then an equivalent way to express the expected excess message length for the
second message part is
\[
L_P(F)\defeq \mathbf{E}_{\mathbf{x}}(R_{F(\mathbf{x})}(\mathbf{x}))=\int_X r(x) R_{F(x)}(x)\text{d}x,
\]
and by minimising the total,
\[
L(F)\defeq L_E(F)+L_P(F),
\]
the optimal $F$ can be found.

Necessarily, the optimal $F$ must have a finite $H(F(\mathbf{x}))$. It
must therefore map to only a countable subset of $\Theta$.
Functions $F:X\to\Theta$ that take only countably many values are known in the
MML literature as \emph{code-books}. The code-books that minimise $L$ are the
\emph{SMML code-books}, and are traditionally taken to provide the SMML
estimate.\footnote{Some sources define ``code-book'' as merely the collection
of $\theta$ values taken by $F$, rather than as the function $F$ itself, but
this is not the sense in which we use the term here.}

Note, however, that SMML code-books may not be unique, and two equally optimal
$F$ functions may lead to distinct estimates. To resolve this, we define the
SMML estimator here more generally as
\[
\hat{\theta}_\text{SMML}(x)=\textit{closure}\left(\bigcup_{F\in\argmin_{F'} L(F')} F(x)\right),
\]
where $\textit{closure}(\cdot)$ is the set closure function.

\subsection{The Ideal Point}\label{SS:ip}

We introduce the notion of an ``ideal point'' which will be central to our
analysis. This is built on an approximation for SMML known in the MML
literature as Ideal Group \citep{Wallace2005}.

The Ideal Group estimator is defined in terms of its functional inverse,
mapping $\theta$ values to (sets of) $x$ values. We refer to such functions
as \emph{reverse estimators} and denote them $\tilde{x}(\theta)$.
The Ideal Group reverse estimator is defined as
\begin{equation}\label{Eq:IG}
\tilde{x}_{\text{IG}}(\theta)\defeq \{x\in X|R_{\theta}(x)\le t(\theta)\},
\end{equation}
where $t(\theta)$ is a threshold whose value is given in \citet{Wallace2005},
and which is computed in a way that guarantees that
the ideal group is a non-empty set for each $\theta\in\Theta$.

Because the ideal group is always non-empty, it must include
\[
\tilde{x}_{\text{IP}}(\theta)\defeq\argmin_{x\in X} R_{\theta}(x).
\]
We refer to this as the \emph{Ideal Point} approximation (a notion and a name
that, unlike Ideal Group, are new to this paper).

We denote the inverse functions of reverse estimators, e.g.
\[
\hat{\theta}_{\text{IP}}(x)\defeq\{\theta\in\Theta|x\in\tilde{x}_{\text{IP}}(\theta)\},
\]
by the same hat notation as estimators,
but stress that these are only true estimators (albeit, perhaps, multi-valued)
if the reverse estimator is a surjection.

To motivate the ideal point estimator independently of the ideal group,
consider that $R_\theta(x)$ is
a representation-invariant value that measures the joint probability density
of $(\boldsymbol{x},\boldsymbol{\theta})$ at $(x,\theta)$, as normalised by
the individual marginal densities of $\boldsymbol{x}$ at $x$
and $\boldsymbol{\theta}$ at $\theta$. (One can think of it as the negative
of the portion of the mutual information of $\boldsymbol{x}$ and
$\boldsymbol{\theta}$, $I(\boldsymbol{x};\boldsymbol{\theta})$, relating
to their distributions at $\boldsymbol{x}=x$ and $\boldsymbol{\theta}=\theta$.)
The
reverse estimator $\tilde{x}_{\text{IP}}(\theta)$ is therefore the estimator
that chooses for any given $\theta$ the $x$ values that are, in this
sense, most ``representative''---or, otherwise stated, are
``prototypical examples''---of $\theta$; they are the $x$ values that
would have maximised the likelihood function $f(x|\theta)$ had $\boldsymbol{x}$
been reparameterised so that its marginal distribution were to become uniform.

The (forward) estimator
$\hat{\theta}_{\text{IP}}(x)$ can by this be understood to select for any $x$
the $\theta$ value(s) for which this $x$ is a prototypical example.

As the observation space is typically much larger (has higher dimensionality)
than the parameter space, it is usually the case that typical $\theta$ values
map to many $x$ values that can serve as prototypical examples of them.
However, optimally we would hope the inverse relationship, being the IP
estimator, is a functional relationship, mapping each example back to a single
antecedent: optimally, an observation is prototypical of exactly one
choice of parameter values. When this is the case, the IP estimator is,
at least at face value, an appealing choice for an
estimator. Conversely, in cases where the relationship is more complicated
we might consider the estimator less justified, and the farther
from functional the relationship is the weaker the justification.

\section{Context, historical background, and aims}\label{S:context}

Before delving into proofs, it is illuminating to place our present
results within their proper context.

\subsection{The operational significance of MML}

In classical (frequentist) statistics, a key tenet is the
likelihood principle \citep{birnbaum1962foundations}. This idea, whose
origins go back as far as \citet{fisher1922mathematical}, states that all
evidence relevant for the modelling of parameters from observations is
contained in the likelihood function. This is the philosophical justification
behind likelihood maximisation.

In practice, however, when comparing between qualitatively different models,
such as between models whose number of degrees of freedom are different,
relying solely on likelihood maximisation is biased towards the more expressive
models, leading to overfitting. This gives rise to multiple types of
penalised maximum likelihood, some, based on likelihood-regularisation,
of the form
\[
\hat{\theta}(x)=\argmax_{\theta\in\Theta} \left[f(x|\theta)-\lambda(\theta)\right],
\]
and some, based on information theory, of the form
\begin{equation}\label{Eq:info_theory}
\hat{\theta}(x)=\argmin_{\theta\in\Theta} \left[-\log f(x|\theta)+\lambda(\theta)\right],
\end{equation}
where different methods use different $\lambda(\theta)$ functions to encode
a penalty for the choice of $\theta$.

This method has proved itself useful in regularisation
\citep{tikhonov2013numerical,hoerl1970ridge}, variable selection
\citep{tibshirani1997lasso}, model selection
\citep{Akaike1973, schwarz1978estimating} and elsewhere.

While all information-based methods rely on
the principles of algorithmic and statistical information theory,
in some cases, such as with the Akaike Information Criterion \citep{Akaike1973}
and the Bayes Information Criterion \citep{schwarz1978estimating}, the choice of
$\lambda(\theta)$ function is a generic one, based on generic assumptions,
whereas in others, such as in Solomonoff induction \citep{solomonoff1964formal}
and MDL \citep{rissanen1989stochastic}, the idea is to use the full problem
description to create a specific $\lambda(\theta)$ function. The
latter outlook lends itself naturally to
describing \eqref{Eq:info_theory} as representing the length of a
two-part message with an explicit structure.

MML, while structurally similar to these methods in advocating for a
two-part message whose optimality is rooted in statistical information theory,
is nevertheless distinct in its approach and aims. While the methods above
can be categorised as either frequentist (make no use of a prior) or
objective Bayesian (make use of a prior that is universal, or otherwise is
systematically-derived as a function of the likelihood function), MML is
purely (subjective) Bayesian, relying on a user-provided prior, and rejects
the likelihood principle entirely. Unlike MDL, which can be viewed as
related to this family of penalised maximum likelihood estimators, and is
therefore of importance mainly in model selection and similar tasks
fitting for this family (as can be seen, e.g., in how it is framed by its
own inventor in \citet{rissanen2007information}), MML is painted by its
originator in a vastly different light, namely as the one, singularly best
systematic method for inductive inference (as evidenced in
\citet{Wallace2005}), and as a do-all-end-all solution.

Accordingly, while MDL has progressed and evolved in order to
address new problem fields and problem types (for example, from the
``old MDL'' of \citet{rissanen1989stochastic} to the newer MDL and the
Normalised Maximum Likelihood of \citet{rissanen2007information}), MML has
not, with the minimum message length principle retaining precisely its
original form since its inception in 1968 \citep{WallaceBoulton1968} and
SMML since its introduction in \citet{wallace1975invariant}.

The MML literature, at large, has followed in this spirit, and much of it
revolves pitting MML in specific tasks against methods such as the ones
listed above, as well as against ML, MAP and other alternatives, demonstrating
MML's distinctive power as a single, unified framework within which to make all
types of statistical decisions.

Examples of such use begin already at MML's inception in
\citet{WallaceBoulton1968}, where MML was used for non-parametric clustering
and unsupervised, non-parametric mixed-modelling,
producing an algorithm that can handle seamlessly both discrete and
continuous attributes, and where the minimum message length criterion is
used to determine how finely one can confidently separate clusters. The
program developed implementing this algorithm, SNOB, is still in use
\citep{dale2016environmental, allison2018coding}, despite now being over 50
years old.

Elsewhere, MML was used to tackle the problem of learning from small
samples, and in particular short time series \citep{schmidt2016minimum},
to handle spike-and-slab priors \citep{xu2017bayesian} (as well as
generalisations thereof, such as the gap problem
of \citet{DoweGardnerOppy2007}), in situations where
the topology of observation- or parameter-space is not trivial (such as
when using phase parameters
\citep{WallaceDowe2000,kasarapu2015minimum,dowe1996mml} or
when the observation- and parameter-spaces have a complex internal structure
\citep{ComleyDowe2005}), for single and multiple factor analysis
\cite{EdwardsDowe1998,wallace1995multiple} and more.

In practical settings, too, MML has been used in a wide variety of domains,
demonstrating empirically its versatility and strength, with recent
examples including
\citet{zamzami2018mml,li2018subspace,channoufi2018color,subramanian2017statistical,konagurthu2019information}.

Throughout, MML was used in its original formulation (up to approximations
that allow its practical computation), cementing within the MML community
its reputation as a statistical Swiss army knife for all occasions.

\subsection{Importance of the consistency question}

Consistency was first introduced by Fisher \citep{fisher1922mathematical},
though its meaning has since evolved \citep{Gerow1989}.
While it was arguably a property of only limited interest upon its initial
introduction, in the modern world of Big Data its impact has increased
exponentially: as part of the present standard practice of machine learning,
we bombard our algorithms with as much data as we can possibly give them,
often at extremely high costs relating to data collection, management,
storage, cleansing, safe-keeping, matching, re-formatting and governance.
Accordingly, it has now become
imperative to ascertain whether the algorithms we use
can make good use of this additional data, or whether their learning curve
will at some point level off. For an inconsistent algorithm, we would like
to know when such saturation occurs, whereas for a consistent algorithm, we
would like to know how efficient it is, i.e.\ how quickly it converges with
data. In both cases, these answers determine how much data can be used
effectively, providing a bound for our collection needs.

Determining whether MML is consistent for the Neyman-Scott scenario,
specifically, is important for two reasons.

Firstly,
the importance of the Neyman-Scott problem stems simply from the fact that
it is a commonly-encountered statistical situation for which solutions are
required in practice. It is the perhaps-simplest statistical example
of a wider class of problems in which additional parameters are added as
more observations become available, a situation that is common and of
practical importance. For example, it is encountered regularly in the
analysis of panel data \citep{baltagi2008econometric}, where it can appear
due to the presence of individual effects, confounders and intermediate
variables. The general problem has been given many names such as
the problem of nuisance variables or the problem of incidental parameters.
Basu \citep{basu2011elimination} describes it as ``The big
question in statistics'', remarking that ``During the past seven decades
an astonishingly large amount of effort and ingenuity has gone into the
search for reasonable answers to this question''.

Second, analysis of the Neyman-Scott scenario is important specifically in
the context of MML because of how it was handled elsewhere in statistics.
Lancaster \citep{lancaster2000incidental}, which provides a full review
of the topic, summarises the state of the art by stating that solution
approaches ``are advanced on a case by case basis, typically these involve
differencing, or conditioning, or use of instrumental variables''.

From an MML perspective, such case-by-case solutions are inadequate,
in that MML purports to be a single,
systematic method that should be able to handle all statistical situations
without requiring any ``tweaking'' for special cases of interest.
Thus, the question of whether Strict MML, in its plain-vanilla form, is
robust enough to handle a Neyman-Scott scenario becomes of interest.

\subsection{Is consistency for Neyman-Scott feasible?}

Maximum likelihood has known good consistency properties
\citep{Doob1934,Wald1949,Perlman1972,SeoLindsay2013}, but they are not
absolute. In cases where ML is not consistent, surely, it would be difficult
for penalised ML methods to guarantee consistency. However, the same
logic cannot be applied to methods that use additional information, e.g.\ in
the form of a Bayesian prior. As an example, in the context of proper estimation
problems with a discrete parameter space, MAP has been shown to be consistent
for every estimation problem that has a consistent posterior, which, in turn,
is every estimation problem for which any consistent estimator exists, and
is, by this, therefore strictly superior to ML in its consistency properties
over this domain \citep{Hendrey2019hierarchy}.\footnote{When discussing
estimation problems with a discrete parameter space, MAP, or
\emph{discrete MAP}, is taken to be the estimator maximising the posterior
\emph{probability} for $\boldsymbol{\theta}=\hat{\theta}_{\text{MAP}}(x)$. This
is in contrast to the more popular use of the term MAP, which is used in the
context of continuous estimation problems (estimation problems with a
continuous hypothesis space) to indicate the estimator
maximising the posterior \emph{probability density},
$f(\theta|x)$.}

The Neyman-Scott problem, being neither discrete nor (in the cases discussed
here) with a proper prior, and having, additionally, an inconsistent
posterior, does not guarantee MAP's consistency, and, in fact, MAP returns for
it estimates that are even more biased than those of ML. This, however,
does not mean that consistent estimation of $\sigma$ is not possible
in this case.

While in classical, frequentist statistics the likelihood principle
suggests that no estimation method can do better than
(penalised) maximum likelihood, in Bayesian statistics no equivalent principle
exists. Different methods, operating on different principles, are free to
use the prior information in different ways and each method may have strengths
and weaknesses related to its choices, in a way that makes them in general
incomparable. Consider, for example, the estimator returning the posterior
expectation of $\theta$,
\[
\hat{\theta}_{\text{PE}}(x)=\mathbf{E}[\boldsymbol{\theta}|\boldsymbol{x}=x].
\]
This estimator does not have, in general, strong consistency properties. It
can be inconsistent even on proper, discrete estimation problems with
independent, identically distributed (i.i.d.) observables, even in cases where
the posterior is consistent and
both MAP and ML return consistent estimates.\footnote{As an example, consider
the case where for any natural $k$ we have $\theta=k!$ as a parameter choice,
with prior probability $h(\theta)=1/2^k$ and the distribution of each
observation given $\theta$ being independent and uniform over $\{1,\ldots,k\}$.}
Nevertheless, it displays remarkable consistency
properties in Neyman-Scott-like situations. For the Neyman-Scott problem itself,
posterior expectation is a consistent estimator under all priors that are
not pathological (in the sense defined in Section~\ref{SS:notations}).

Thus, because MML is distinct from other popular information-based estimation
methods in that it is wholly (subjective) Bayesian, it is not, itself, an
instance of penalised maximum likelihood. Any consistency or inconsistency
results regarding maximum likelihood cannot, therefore, be used to infer
directly regarding MML's consistency. The only way to determine the
consistency of MML is to consider its properties specifically, and there is
no reason to assume, without such MML-specific evidence, whether it 
is likely to be consistent on a problem such as Neyman-Scott, or, in fact,
in general.

\subsection{The history of MML's consistency analysis}

Surprisingly, perhaps, SMML was not initially introduced as part of
statistical information theory. Its introduction, in 1975
\citep{wallace1975invariant}, was as an attempt to
create a new estimator, primarily for estimation problems with a
continuous hypothesis space, that
will extend the good properties exhibited by MAP in the discrete domain also
to the continuous domain. The typical continuous-domain analogoue to discrete
MAP, i.e.\ the maximisation of the posterior probability density,
was deemed unsatisfactory in this context, because it leads to
estimations that are not invariant to the representation of the parameter
space, which loses some of discrete MAP's key properties.

Two properties that are explicitly mentioned in \citet{wallace1975invariant}
as part of this
intended goal, which are purposefully preserved by SMML, are its Bayesian
nature and its invariance to representation.
However, implicitly, the hope was to preserve as
many other of MAP's good properties as well. In particular, as discussed, MAP
in the discrete space has optimal consistency properties, and it was therefore
reasonable to intuitively expect (or at least hope for) such properties to
extend also to SMML in the continuous domain.

Wallace and Boulton conclude \citet{wallace1975invariant}
by demonstrating that SMML, as derived, can be
equivalently---and more elegantly---described in information-theoretic
terms, thus linking SMML to the statistical information theory literature,
where it has resided ever since. Over the following years additional
results, such as Solomonoff's \citep{Solomonoff1978} proofs of convergence,
provided additional evidence regarding the consistency powers of
the wider family of information-theoretic methods, and by this also
indirectly regarding MML.

The most persuasive affirmation of this fact came in \citet{BarronCover1991},
where it was proved regarding general information-theoretic estimation methods,
including Solomonoff induction, MDL, MML and others, that they satisfy certain
general consistency properties simply by virtue of their structure.

The question remained, however, whether in the ostensibly more difficult case
of a Neyman-Scott scenario, not covered by the proofs of
\citet{BarronCover1991},
it would still be the case that Strict MML retains its consistency properties.

The intuition that it should is one that can equally be applied to the entire
two-part-message information-theoretic estimator family. In all cases, the
idea is that any information about the observation, $x$, that can be
meaningfully
formulated as a pattern is encoded in the first part of the message, and
therefore reflected by the choice of $\hat{\theta}(x)$, whereas the second
part of the message encodes the noise. Given enough observations, it is
argued (e.g., in \citet{DoweGardnerOppy2007}),
any pattern that can be discerned will ultimately
be captured by the first part of the message and therefore reflected in
$\hat{\theta}(x)$, ensuring its consistency.

Specifically, one can think of this two-part message structure as a built-in
protection against over-fitting, with the method continuously weighing how
much of the information it receives is salient and how much is not. The
Neyman-Scott problem is, in this context, an optimal test to showcase the
powers of MML, in that at its core it is a problem regarding over-fitting:
at any given point, the uncertainty regarding $\mu$ is high, but
point estimation methods following a maximum likelihood principle will
choose to estimate its value as equal to $m$; this is an over-fit for $\mu$
that results in an underestimated $\sigma$. If MML's protection against
over-fitting is absolute however, it can be argued, such biases should not
occur in its estimations.

When \cite{dowe1997resolving} and \cite{Wallace2005} showed theoretically
regarding two of SMML's
approximations that they are consistent for a certain Neyman-Scott problem,
and when \cite{dowe1996mml} and \cite{WallaceDowe1993} further showed
empirically that WF-MML converges
properly, while other estimation methods do not, in other Neyman-Scott-like
scenarios, this provided significant theoretical and empirical validation for
the overarching theory regarding MML's capabilities. Because SMML is NP-hard
in general, and no efficient algorithm to compute it is known
even in the simplest of multidimensional cases\footnote{In
\citet{farrwallace2002}, for example, the problem of determining the SMML
estimate to the parameters
of a trinomial distribution is investigated; the authors write that they
have not found any polynomial-time algorithm for this problem, have no
non-trivial bounds on its complexity, and suspect it to be
NP-hard.}, such validation
was the best that MML researchers could hope for, and the results were
therefore understandably accepted at face value.

Over the years since these results began to emerge, MML's consistency
properties were, consequently, repeatedly cited without any form of caveat,
and were firmly believed among the MML community. A particularly striking
example of this is \citet{dowe2011mml}, where Dowe observes
that while some estimation methods, like posterior expectation, provide
consistency in Neyman-Scott-like scenarios, and while other methods, like
maximum likelihood, provide an estimate that is invariant to the representation
of both parameter space and observation space, no method other than MML
is known to
provide both absolute consistency and such invariance. Dowe conjectured that
only MML and very closely-related Bayesian methods are in general both
statistically consistent and invariant, adding a back-up conjecture that
if there are (hypothetically) any such non-Bayesian methods, they will be
far less efficient than MML.
SMML's own consistency was, throughout, never in question, and this reflects
the views of the MML community at large.

It is in light of these facts that proving that SMML is, in fact, not
consistent even for the original and much-showcased Neyman-Scott problem
becomes highly significant: it changes our understanding of MML's place in
statistical theory, proving that it is not the universally applicable (albeit
computationally intractable) statistical tool that it was thought to be, and
that it does require case-specific tweaking in order to handle real-world
scenarios of interest, even under conditions of unlimited data.

This new understanding regarding the true powers and the true limitations of
SMML opens up entirely new research questions, or, more precisely, reawakens
old research questions that have remained untouched since the introduction of
the minimum message length principle over 50 years ago, and which can now
once again be investigated.

For example, there is no known result, not even among the results of
\citet{BarronCover1991}, that proves that SMML is consistent even in the
ostensibly simplest case of proper, discrete estimation problems with a
consistent posterior, even where both MAP and maximum likelihood are consistent.
Determining the answer to this question, one way or another, would be a
significant result for MML theory.

\subsection{Aims of this paper}

Despite the fact that this paper's main technical proof is regarding MML's
consistency for the Neyman-Scott problem, and despite the fact that the
wider context was presented, correspondingly, along these lines, it should
be noted that neither the paper's main contribution nor its main aim has
to do with either the Neyman-Scott problem or consistency.

To explain, consider this.

Due to its unique theoretical underpinnings, Strict MML
is widely believed to hold many good properties that set it apart from
other point estimation methods, and much of the MML literature, both
theoretical and experimental, deals with comparing MML with other, more
commonly-used solutions to the same problems, and reporting on MML's
advantages.

Unfortunately, because MML is almost invariably computationally intractable,
papers demonstrating SMML’s good properties typically do so by relying on
approximations, additional assumptions and analysis of specific, narrow cases.

Where these demonstrations are successful, they can provide further evidence
for SMML's power, but not conclusive proofs, as the additional assumptions
remain suspect. Where they are unsuccessful, they cannot prove any deficiency
in SMML, as the problem may well be in the approximations used.
Thus, the inability to generate exact SMML solutions has been a great
hindrance for the study of the true properties and power of MML.

This paper's main contribution is that it develops, for the first time,
means to analyse SMML directly, without approximations or assumptions, on a
wide range of problems, including high-dimensional natural ones.

Some of our conclusions, such as Lemma~\ref{L:Vmax}, prove certain basic
good properties regarding SMML estimates. However, by and large the methods
we use define problem scenarios in which SMML coincides with more standard
estimation methods, particularly maximum likelihood estimation. As such,
these methods are not usable to demonstrate SMML's edge over competing
methods.

Where our main aim lies, and is the main impact of our contribution, is in
providing systematic new tools that can, at least, provide \emph{negative}
results regarding the power of MML: where prior methods have used assumptions
and approximations to differentiate MML from its competitors, and by this
provided evidence (but not proof) of MML's power, we can test these claims
rigorously, and potentially negate them, if we find a problem satisfying our
required conditions within the domain regarding which MML's superiority
is claimed, because for this domain we know that no such differentiation
exists.

Our result regarding MML's lack of consistency in the Neyman-Scott scenario,
however central and influential for MML, is in this light merely a first
demonstration of the power of this new, general tool.

\section{SMML analysis}\label{S:smml}

In this section we describe our SMML analysis methods.

Because of SMML's inherently NP-hard nature, it is not possible to describe
exact SMML solutions
for arbitrary estimation problems, and so, instead, we focus on defining
properties of estimation problems, such that the SMML solution for problems
exhibiting these properties can be analysed exactly. These properties,
defined below, include \emph{transitivity}, \emph{homogeneity},
\emph{concentration} and \emph{locality}
Our various theorems require estimation problems to exhibit
various subsets of these. (The intersection of all properties required for
all our results we've named \emph{regularity}.)
Thus, these properties have no operational
justification of their own, nor can they be individually motivated. Their
definitions relate only to what is needed for the sake of the various proofs.

Within a Bayesian point estimation setting, estimation problems, including
their priors, are considered inputs. They are given, and therefore do not
need to be motivated. The properties listed above should therefore be taken
as descriptive, not prescriptive.

Nevertheless, our aim was to find properties that are exhibited by a wide
range of interesting and/or natural problems, and so the main criterion for
a good property was that it should exclude as few of these as possible.

Some of our properties (concentration, locality) attain these ideals directly,
by only excluding problems exhibiting certain pathological behaviours.
The definition of locality, for example, does not exclude any proper
estimation problem at all. Others (transitivity, homogeneity) rely on internal
symmetries within the estimation problem. They exclude more. However, because
of the aesthetics of such symmetries, it is still the case that many interesting
and relevant problems satisfy them.

We begin by defining these properties.




\subsection{Some special types of inference problems}\label{SS:types}

The first few properties we require of estimation problems
describe symmetries, i.e.\ automorphisms, which we will exploit in
constructing the SMML solution.

\begin{defi}\label{D:transitivity}
An \emph{automorphism} for
an estimation problem $(\mathbf{x},\boldsymbol{\theta})$, with
$\mathbf{x}\in X$ and $\boldsymbol{\theta}\in\Theta$, is a pair $(U,T)$ of
diffeomorphisms, $U:X\to X$ and $T:\Theta\to\Theta$, such that
\begin{enumerate}
\item For every $A\subseteq X$,
\begin{equation}\label{Eq:automorphism_marginal}
\text{ScaledProb}(\mathbf{x}\in A)=\text{ScaledProb}(\mathbf{x}\in U(A)),
\end{equation}
and
\item For every $A\subseteq X$ and every $\theta$,
\begin{equation}\label{Eq:automorphism_likelihood}
\text{Prob}(\mathbf{x}\in A|\theta)=\text{Prob}(\mathbf{x}\in U(A)|T(\theta)),
\end{equation}
\end{enumerate}
where $U(A)=\{U(y)|y\in A\}$.

Note that we assume that $U$ and $T$ are such
that the Jacobians of these bijections, $\frac{\text{d}U(x)}{\text{d}x}$ and
$\frac{\text{d}T(\theta)}{\text{d}\theta}$, are defined everywhere, and their
determinants, $\left|\frac{\text{d}U(x)}{\text{d}x}\right|$ and
$\left|\frac{\text{d}T(\theta)}{\text{d}\theta}\right|$, are positive
everywhere. This allows us to restate condition
\eqref{Eq:automorphism_marginal} as
\[
r(x)=r(U(x))\left|\frac{\text{d}U(x)}{\text{d}x}\right|,
\]
and condition \eqref{Eq:automorphism_likelihood} as
\[
f(x|\theta)=f(U(x)|T(\theta))\left|\frac{\text{d}U(x)}{\text{d}x}\right|.
\]

An estimation problem will be called \emph{observation transitive}
if for every $x_1,x_2 \in X$ there is an automorphism $(U,T)$
for which $U(x_1)=x_2$.

An estimation problem will be called \emph{parameter transitive}
if for every $\theta_1,\theta_2 \in \Theta$ there is an automorphism $(U,T)$
for which $T(\theta_1)=\theta_2$.

An estimation problem will be called \emph{transitive} if it is both
observation transitive and parameter transitive.
\end{defi}

Here, we borrow the term ``transitivity'' from graph theory, where it is
used, in the context of vertex-transitive and edge-transitive graphs, to
describe graph properties defined by analogous automorphisms.

\begin{lemma}\label{L:NStransitive}
The scale free Neyman-Scott problem with fixed $N$ and $J$ and with
observable parameters $(s,m)$ is transitive.
\end{lemma}

\begin{proof}
Consider $U(s,m)=(\alpha s,\alpha m+\Delta)$ and
$T(\sigma,\mu)=(\alpha \sigma,\alpha \mu+\Delta)$ with $\alpha> 0$.
It is straightforward to verify that $(U,T)$ is an automorphism.
Furthermore, for any $(s,m)$ and $(s',m')$
it is straightforward to find parameters $\alpha$ and $\Delta$ that would
map $(s,m)$ to $(s',m')$, and similarly for $(\sigma,\mu)$ and
$(\sigma',\mu')$.\footnote{We refer to estimation problems as \emph{scale free}
if they admit an
automorphism $(U,T)$ such that $U(x)=\alpha x$, i.e.\ if  the description of
the problem's likelihoods and prior would not have changed if all $x$ values
had been given in different units of scale, assuming that we make a
corresponding change also to the representation of $\boldsymbol{\theta}$.
As can be seen, the scale-free Neyman-Scott problem
with observables $(x_{nj})$ admits such an automorphism, and hence its name.

This, however, is not a property we will use in this paper, and its definition
appears here solely to explain the naming choice for the
scale-free Neyman-Scott problem and its prior.}
\end{proof}

Transitivity is a strong condition on the symmetry of a problem: any two
possible observables must be symmetric to each other (i.e., have automorphisms
mapping one to the other) and any two parameter choices must be symmetric to
each other. We will, for the most part, only require weaker conditions,
relating to all parameter choices/observations being in some sense ``similar''
to each other specifically with respect to the $R_{\theta}(x)$ function.

Formally, we define the weaker conditions that we need as follows.

\begin{defi}\label{D:homogeneous}
Let
\[
R^*_\theta\defeq\min_{x\in X} R_{\theta}(x).
\]

An estimation problem $(\mathbf{x},\boldsymbol{\theta})$, with
$\mathbf{x}\in X$ and $\boldsymbol{\theta}\in\Theta$, will be called
\emph{parameter-homogeneous} if the value of $R^*_\theta$
is a constant, $R^*$, for all $\theta\in\Theta$.

Let
\[
R_{\textit{opt}}(x)\defeq\min_{\theta\in \Theta} R_{\theta}(x).
\]

An estimation problem $(\mathbf{x},\boldsymbol{\theta})$, with
$\mathbf{x}\in X$ and $\boldsymbol{\theta}\in\Theta$, will be called
\emph{observation-homogeneous} if the value of $R_{\textit{opt}}(x)$
is a constant, $R_{\textit{opt}}$, for all $x \in X$.

An estimation problem will be called \emph{homogeneous} if it is both
parameter-homogeneous and observation-homogeneous.
\end{defi}

\begin{lemma}\label{L:homogenous}
Every parameter-transitive estimation problem is parameter-homogeneous.

More generally, for any $\theta_1, \theta_2\in\Theta$, if there exists an
automorphism $(U,T)$ such that $T(\theta_1)=\theta_2$, then
$R^*_{\theta_1}=R^*_{\theta_2}$.
\end{lemma}

\begin{proof}
Assume to the contrary that for some such $\theta_1, \theta_2\in\Theta$, the
inequality
$R^*_{\theta_1}>R^*_{\theta_2}$
holds.

Let $(U,T)$ be an automorphism on $(X,\Theta)$ such that $T(\theta_1)=\theta_2$,
and let $x\in X$ be a value such that $R_{\theta_2}(\cdot)$ attains its minimum
at $U(x)$.

By definition,
\begin{align*}
R^*_{\theta_1} &\le R_{\theta_1}(x)
=\log\left(\frac{r(x)}{f(x|\theta_1)}\right) \\
&=\log\left(\frac{r(U(x))\left|\frac{\text{d}U(x)}{\text{d}x}\right|}{f(U(x)|T(\theta_1))\left|\frac{\text{d}U(x)}{\text{d}x}\right|}\right) \\
&=R_{\theta_2}(U(x))
=R^*_{\theta_2},
\end{align*}
contradicting the assumption.

The option $R^*_{\theta_1}<R^*_{\theta_2}$ also
cannot hold, because $(U^{-1}, T^{-1})$ is also an automorphism, this one
mapping $\theta_2$ to $\theta_1$.
\end{proof}

Similarly:

\begin{lemma}\label{L:comprehensive}
Every observation-transitive estimation problem is observation-homogeneous.

More generally, for any $x_1,x_2\in X$ for which there exists an automorphism
$(U,T)$ such that $U(x_1)=x_2$,
$R_{\textit{opt}}(x_1)=R_{\textit{opt}}(x_2)$.
\end{lemma}

\begin{proof}
The proof is identical to the proof of Lemma~\ref{L:homogenous},
except that instead of choosing $(U,T)$ such that $T(\theta_1)=\theta_2$
we now choose an automorphism such that $U(x_1)=x_2$, and instead of
choosing $x\in X$ such that $R_{\theta_2}(\cdot)$ attains its minimum at $U(x)$,
we choose $\theta \in \Theta$ such that $R_{\theta}(x_2)$ attains its minimum
over all $\theta$ at $T(\theta)$.
\end{proof}

We now introduce two other properties, which are expectable of typical,
natural problems. These are properties that we require merely to exclude
potential pathological behaviours of the problem (and of estimators on it).
The first of these is \emph{concentration}.

Define for every $\epsilon>0$,
\[
\tilde{x}_{\epsilon}(\theta)\defeq\{x\in X | R_{\theta}(x)-R^*_\theta<\epsilon\},
\]
and
\[
\hat{\theta}_{\epsilon}(x)\defeq\{\theta\in\Theta|x\in\tilde{x}_{\epsilon}(\theta)\}.
\]

\begin{defi}
An estimation problem $(\mathbf{x},\boldsymbol{\theta})$ will be called
\emph{concentrated} if for every $x \in X$ there is an $\epsilon>0$
for which $\hat{\theta}_{\epsilon}(x)$ is a bounded set.
\end{defi}

The $\tilde{x}_\epsilon$ reverse estimator and its inverse are generalisations
of the ideal point and ideal group concepts, considering an observation $x$
to be a good representative for $\theta$ if $R_\theta(x)$ equals its optimal
value, $R^*_\theta$, up to a difference of some predetermined margin,
$\epsilon$.

To motivate the definition of concentration and to understand it
intuitively, consider the following.

As discussed in Section~\ref{SS:ip}, the justification for the use of the
Ideal Point estimator is strongest when the inverse of the relationship between
$\theta$ values and their ``prototypical examples'' (in the sense of
$\tilde{x}_{\text{IP}}(\theta)$) is functional, each $x$ value mapping to
exactly one $\theta$ value. A weakening of this
condition would have required the relationship between $\theta$ values and
their ``good representatives'' (which in this context we understand as
the set $\tilde{x}_\epsilon(\theta)$, for some choice of $\epsilon$)
to be such that any observation $x$ can only
be a good representative to some ``closely-related'' set of $\theta$.
The property of being concentrated can be understood as the weakest form of
this condition.  It stipulates that
for each $x$ neither the $\theta$ values for which it is a prototypical
example nor the $\theta$ values for which it is an arbitrarily good
representative (i.e., $\tilde{x}_\epsilon(\theta)$, for an arbitrarily
small $\epsilon>0$) can be unbounded sets. If an estimation problem is not
concentrated, there is little reason to advocate for it the use of the
Ideal Point estimator, as it may return either unbounded sets as estimates,
or its estimates will be arbitrarily far from other almost equally good
parameter choices.

The scale free Neyman-Scott problem, however, meets the ideal requirement
for IP estimation, namely that the relationship between $\theta$ values
and their prototypical examples is, for this problem, a bijection, as is
demonstrated by the following lemma and its proof.

\begin{lemma}\label{L:NSconcentrated}
The scale free Neyman-Scott problem with fixed $N$ and $J$ and with
observable parameters $(s,m)$ is concentrated.
\end{lemma}

\begin{proof}
In the $(x_{nj})$ observation space, the probability density of a given set of
observations, $x$, assigned in the Neyman-Scott problem to a particular choice
of $\sigma^2$ and $\mu$, is
\begin{equation}\label{Eq:NSlikelihood}
f(x |\sigma^2, \mu)
=\frac{1}{(\sqrt{2\pi}\sigma)^{NJ}} e^{-\frac{\sum_{n=1}^{N} \sum_{j=1}^{J}(x_{nj}-\mu_n)^2}{2\sigma^2}}.
\end{equation}

Under the scale-free prior, this results in the marginal probability density
of the observations being
\begin{align}\label{Eq:NSr}
\begin{split}
r( x ) &=\int_{0}^{\infty} \frac{1}{\sigma^{N+1}} \iint_{-\infty}^{\infty} f(x|\sigma^2,\mu) \text{d}\mu_1\cdots\text{d}\mu_n\text{d}\sigma \\
&=2^{N/2-1} J^{-N/2} \pi^{-\frac{N(J-1)}{2}} (NJs^2)^{-\frac{NJ}{2}}\Gamma\left(\frac{NJ}{2}\right).
\end{split}
\end{align}

Note that $(s,m)$ is a sufficient statistic for this problem, because both
$f(x |\sigma^2, \mu)$ and $r(x)$ can be calculated
based on it, where for $f(x |\sigma^2, \mu)$ we use the relation
\[
\sum_{n=1}^{N} \sum_{j=1}^{J}(x_{nj}-\mu_n)^2 = (NJs^2)+J\sum_n (m_n-\mu_n)^2.
\]
For this reason, we can present the equations above solely in terms of $(s,m)$.

Substituting now \eqref{Eq:NSlikelihood} and \eqref{Eq:NSr} into
\eqref{Eq:Rdef}, we get
\begin{align}\label{Eq:R_NS}
\begin{split}
R_{(\sigma^2, \mu)}( x )
=&-\frac{N}{2}\log J+\frac{NJ+N-2}{2}\log 2 +\frac{N}{2}\log\pi \\
&+\log\left(\Gamma\left(\frac{NJ}{2}\right)\right)+\frac{NJ}{2}\log(\sigma^2)+\frac{NJs^2}{2\sigma^2} \\
&+\frac{J}{2\sigma^2}\sum_n (m_n-\mu_n)^2-\frac{NJ}{2}\log(NJs^2).
\end{split}
\end{align}

This is a strictly convex function of $(s^2,m)$, with a
unique minimum, for any $(\sigma^2,\mu)\in\Theta$. As such,
$\tilde{x}_\epsilon(\theta)$ is bounded for any $\epsilon$.

Note regarding the definition of $R$ in \eqref{Eq:Rdef} that it is invariant
to both the representation of the parameter space and the representation of
the observation space. For this reason, the calculated $R_{(\sigma^2,\mu)}(x)$
would be exactly the same as $R_{(\sigma^2,\mu)}(s^2,m)$ under the
parameterization of interest to us.

Consider, now, the Neyman-Scott problem under the
parameterization
$(\log \sigma,\mu/\sigma)$ and $(\log s,m/s)$. In this\linebreak
re-parameterization,
it is easy to see that for any translation function, $T_\Delta(a)=a+\Delta$, 
$(T_\Delta,T_\Delta)$ is an automorphism. In particular, this means that for
any $\theta_0$,
\[
\tilde{x}_\epsilon(\theta_0)=\{x+\theta_0-\theta|x\in\tilde{x}_\epsilon(\theta)\}.
\]
All such sets are translations of each other, having the same volume, shape and
bounding box dimensions.

It follows regarding the inverse function, $\hat{\theta}_{\epsilon}(x)$, that
for any $x$ it maps to a set of the same volume and bounding box dimensions
as each $\tilde{x}_\epsilon(\theta)$, albeit with an inverted shape.

In particular, it is bounded.

Being bounded under the new parameterization is tantamount to being bounded
under the native problem parameterization.
\end{proof}

The last property we wish to mention,
also relating to avoidance of pathological behaviour,
is the following.

\begin{defi}\label{D:local}
An estimation problem $(\mathbf{x},\boldsymbol{\theta})$
will be called \emph{local} if there exist values $V_0$ and $\gamma>1$ such that
for every $\theta\in\Theta$ there exist $\theta_1,\ldots,\theta_k$, such
that for all $x$ outside a subset of $X$ of total scaled probability at
most $V_0$,
\begin{equation}\label{Eq:local}
\gamma k f(x|\theta)<\max_{i\in\{1,\ldots,k\}} f(x|\theta_i).
\end{equation}
\end{defi}

The exact operational justification for the definition of locality is that
provided by Lemma~\ref{L:Vmax}.
Intuitively, however, one can think of locality as a property complementary to
concentration: whereas concentration requires all $\theta$ that have a certain
$x$ as a good representative to be bounded in parameter space, locality
requires all $x$ which are (in a somewhat different sense) good representatives
for a given $\theta$ to be bounded in terms of their total probability.

Essentially, a problem is local if for each $\theta$ one can find a finite
number of surrogates, $\theta_1,\ldots,\theta_k$, such that the set of $x$
values more closely associated with $\theta$ than with any of its surrogates
is bounded in its total probability. In this way, the surrogates can be
thought of as ``localising'' the impact of $\theta$.
In an estimation problem with a
non-pathological parameterization, one can expect to be able to find such
surrogates simply by surrounding the chosen $\theta$. Moreover, as
demonstrated by Lemma~\ref{L:properlocal}, non-locality does not arise at
all in problems that do not have an improper prior.

\begin{lemma}\label{L:properlocal}
Every proper estimation problem is local.
\end{lemma}

\begin{proof}
Consider any estimation problem over a normalised (unscaled) prior, and
consequently also a normalised (unscaled) marginal.

The total probability over all $X$ is, by definition, $1$, so choosing
$V_0=1$ satisfies the conditions of locality.
\end{proof}

\begin{lemma}\label{L:NSlocal}
The scale free Neyman-Scott problem is local.
\end{lemma}

The proof of Lemma~\ref{L:NSlocal} is given in Appendix~\ref{S:NSlocal}.

\begin{defi}
An estimation problem is called \emph{regular} if it is
observation-transitive, parameter-homogeneous, concentrated and local.
\end{defi}

\subsection{Relating SMML to IP}\label{SS:SMML_IP}

We will now show that for regular problems
one can infer from the IP solution to the SMML solution.

Our first lemma proves for a family of estimation problems that the
SMML solutions to these problems do not diverge entirely, in the sense of
allocating arbitrarily high (scaled) probabilities to single $\theta$ values.
Although a basic requirement for any good estimator, no such result was
previously known for SMML.

For a code-book $F$, let
\[
\textit{region}_F(\theta)\defeq\{x|F(x)=\theta\}
\]
be known as the \emph{region of $\theta$ in $F$}.

\begin{lemma}\label{L:Vmax}
For every local estimation problem there is a $V_{\textit{max}}$ such that no
SMML code-book $F$ for the problem contains any
$\theta\in\Theta$ whose region has scaled probability greater than
$V_{\textit{max}}$ in the marginal distribution of $X$.
\end{lemma}

\begin{proof}
Let $V_0$ and $\gamma$ be as in Definition~\ref{D:local}. Note
that $V_0$ can always be increased without violating the conditions of the
definition, so it can be assumed to be positive.

Assign $V_{\textit{max}}=(\beta_0^{-1}+1)V_0$ for a constant $\beta_0>0$ to be
computed later on, and assume for contradiction that $F$
contains a $\theta$ whose region, $X_\theta$, has scaled probability $V$
greater than $V_{\textit{max}}$. By construction, $X_\theta$ contains a
non-empty, positive scaled probability region $X'$ wherein \eqref{Eq:local} is
satisfied.

Let $V_b$ be the scaled probability of $X'$, and let $V_a$ be $V-V_b$.

Also, define $\beta=V_a/V_b$, noting that
\begin{equation}\label{Eq:bb0}
\beta<\beta_0,
\end{equation}
because,
by assumption, $V_a+V_b>V_{\textit{max}}$ and $V_a\le V_0$, so
\[
\beta^{-1} = \frac{V_b}{V_a} > \frac{V_{\textit{max}}}{V_0}-1 = \beta_0^{-1}.
\]

We will design a code-book $F'$ such that $L(F')<L(F)$, proving by contradiction
that $F$ is not optimal.

Our definition of $F'$ is as follows. For all $x\notin X'$, $F'(x)=F(x)$.
Otherwise, $F'(x)$ will be the value among $\theta_1,\ldots,\theta_k$
for which the likelihood of $x$ is maximal.

Recall that
\[
L(F)-L(F')=\left(L_E(F)-L_E(F')\right)+\left(L_P(F)-L_P(F')\right).
\]

Because, by construction, the set $X'$, of scaled probability $V_b$, satisfies
that for any $x\in X'$,
\[
\log f(x|F'(x))-\log f(x|F(x))>\log (\gamma k),
\]
we have
\begin{equation}\label{Eq:LPdiff}
L_P(F)-L_P(F')>V_b \log (\gamma k).
\end{equation}
On the other hand, the worst-case addition in (scaled) entropy caused by
splitting the set $X'$ into $k$ separate $\theta_i$ values is if each
$\theta_i$ receives an equal probability. We can write this worst-case
addition as
\begin{align}\label{Eq:LEdiff}
\begin{split}
L_E(F')-L_E(F) \le & \left[-V_a \log V_a-\sum_{i=1}^{k}\frac{V_b}{k}\log\left(\frac{V_b}{k}\right)\right] \\
& \quad\quad -\left[-(V_a+V_b)\log(V_a+V_b)\right].
\end{split}
\end{align}
This is in the case that $V_a>0$. If $V_a=0$, the expression $V_a \log V_a$
is dropped from \eqref{Eq:LEdiff}. This change makes no difference in the later
analysis, so we will, for convenience, assume for now that $V_a>0$.

Under the assumption $V_a>0$, we can subtract \eqref{Eq:LEdiff}
from \eqref{Eq:LPdiff} to get
\begin{align}\label{Eq:Ldiff}
\begin{split}
L(F)-L(F') & > V_b \log(\gamma k)-V_b \log k -  V_a \log\left(\frac{V_a+V_b}{V_a}\right) \\
& \quad\quad - V_b \log\left(\frac{V_a+V_b}{V_b}\right) \\
&= V_b \log \gamma -  V_a \log\left(\frac{V_a+V_b}{V_a}\right) \\
& \quad\quad - V_b \log\left(\frac{V_a+V_b}{V_b}\right).
\end{split}
\end{align}

To reach a contradiction, we want $L(F)>L(F')$. If $V_a=0$,
equation~\eqref{Eq:Ldiff}
degenerates to $L(F)-L(F') > V_b \log \gamma \ge 0$ for an immediate
contradiction. Otherwise, contradiction is reached if
\[
V_b \log \gamma -  V_a \log\left(\frac{V_a+V_b}{V_a}\right) - V_b \log\left(\frac{V_a+V_b}{V_b}\right) \ge 0,
\]
or equivalently if
\begin{equation}\label{Eq:b1}
\beta\log(\beta^{-1}+1)+\log(\beta+1)\le \log \gamma.
\end{equation}

A small enough $\beta$ value can bring the left-hand side of \eqref{Eq:b1}
arbitrarily close to $0$, and in particular to a value lower
than $\log \gamma$ for any $\gamma>1$.

By choosing a small enough $\beta_0$, we can ensure than any $\beta$
satisfying \eqref{Eq:bb0} will also satisfy \eqref{Eq:b1},
creating a contradiction and proving our claim.
\end{proof}

Lemma~\ref{L:Vmax} now allows us to draw a direct connection between SMML and
$\tilde{x}_{\epsilon}(\theta)$.

\begin{thm}\label{T:approx}
In every local, parameter-homogeneous estimation problem
$(\mathbf{x},\boldsymbol{\theta})$, for every SMML code-book $F$
and for every $\epsilon>0$
there exists a $\theta_0\in\Theta$ for which the set
\[
\textit{region}_{F}(\theta_0)\cap \tilde{x}_{\epsilon}(\theta_0)
\]
is a set of positive scaled probability in the marginal distribution of $X$.
\end{thm}

\begin{proof}
Suppose to the contrary that for some $\epsilon$, no element
$\theta_0\in\Theta$ is mapped from a positive scaled probability of $x$
values from its respective $\tilde{x}_{\epsilon}(\theta_0)$.

Let $\Theta^*\subseteq\Theta$ be the set of $\theta$
values with positive scaled probability regions in $F$, and
let $G$ be the directed graph whose vertex set is $\Theta^*$ and which contains
an edge from $\theta_1$ to $\theta_2$ if the intersection
\[
\tilde{x}_{\epsilon/2}(\theta_1) \cap \textit{region}_F(\theta_2)
\]
has positive scaled probability.
By assumption, $G$ has no self-loops.

Let
\[
V(\theta)=\text{ScaledProb}(\mathbf{x}\in \textit{region}_F(\theta)).
\]

We claim that for any $(\theta_1,\theta_2)$ that is an edge in $G$,
\begin{equation}\label{Eq:v2v1}
\log V(\theta_2)-\log V(\theta_1)\ge\epsilon/2,
\end{equation}
an immediate consequence of which is that
$V(\theta_2)>V(\theta_1)$ and therefore $G$ cannot have any cycles.

To prove \eqref{Eq:v2v1}, note first that because of our assumption that
all likelihoods are continuous, $\tilde{x}_{\epsilon/2}(\theta)$, for every $\theta$
and any choice of $\epsilon>0$,
has positive measure in the space of $X$, and because of our assumption that
all likelihoods are positive, a positive measure in the space of $X$ translates
to a positive scaled probability. This also has the side effect that all
vertices in $G$ must have an outgoing edge (because this positive scaled
probability must be allocated to some edge).

Next, consider how transferring a small subset of $X$, of size
$\Delta$, in $\tilde{x}_{\epsilon/2}(\theta_1) \cap \textit{region}_F(\theta_2)$
from $\theta_2$
to $\theta_1$ changes $L()$. Given that $\Delta$ can be made arbitrarily small,
we can consider the rate of change, rather than the magnitude of change: for
$F$ to be optimal, we must have a non-negative rate of change, or else
a small-enough $\Delta$ can be used to improve $L()$.
Given that $L_E$ is the sum of
$-V(\theta^*) \log V(\theta^*)$ over all $\theta^*\in\Theta^*$,
by transferring probability from $\theta_2$ to $\theta_1$, the rate of change
to $L_E$ is $\log V(\theta_2)-\log V(\theta_1)$.

Consider now the rate of change to $L_P$.
By transferring probability from $\theta_2$,
where it is outside of $\tilde{x}_{\epsilon}(\theta_2)$ (and therefore by definition
assigned an $R$ value of at least $R^*+\epsilon$) to $\theta_1$,
where it is assigned into $\tilde{x}_{\epsilon/2}(\theta_1)$ (and therefore by
definition assigned an $R$ value that is smaller than $R^*+\epsilon/2$)
the difference is a reduction rate greater than $\epsilon/2$.

The condition that the rate of change of $L=L_E+L_P$ is nonnegative
therefore translates simply to \eqref{Eq:v2v1}, thus proving the equation's
correctness.

However, if $G$ contains no self-loops and no cycles, and every one of its
vertices has an outgoing edge, then it contains arbitrarily long paths
starting from any vertex.
Consider any such path starting at some $\theta_1$ of length greater than 
$2[\log(V_{\textit{max}})-\log V(\theta_1)]/\epsilon$,
where $V_{\textit{max}}$ is as in Lemma~\ref{L:Vmax}.
By \eqref{Eq:v2v1}, we have that the scaled probability
assigned to the $\theta$ value ending the path is greater than
$V_{\textit{max}}$, thereby reaching a contradiction.
\end{proof}

We can now present our main theorem, formalising the connection between the
SMML estimator and the ideal point.

\begin{thm}\label{T:SMML_IP}
In any regular estimation problem, for every $x\in X$,
\begin{equation}\label{Eq:SMML_IP}
\hat{\theta}_{\text{SMML}}(x) \cap \hat{\theta}_{\text{IP}}(x) \ne \emptyset.
\end{equation}

In particular, $\hat{\theta}_{\text{IP}}$ is a true estimator, in the sense
that $\hat{\theta}_{\text{IP}}(x) \ne \emptyset$ for every $x\in X$.
\end{thm}


\begin{proof}
%
%
Let $x$ be a value for which we want to prove \eqref{Eq:SMML_IP}.

From Theorem~\ref{T:approx} we know that for all $\epsilon$ there exists
an SMML code-book $F'$ and a $\theta^*\in\Theta$ for which
$\textit{region}_{F'}(\theta^*)\cap \tilde{x}_{\epsilon}(\theta^*)$
is non-empty.
Let $x_0$ be a value inside this intersection.

By observation transitivity, there is an automorphism $(U,T)$ such that
$x=U(x_0)$. Let $\theta_\epsilon=T(\theta^*)$.

Let us define $F$ by $F=U^{-1}\circ F' \circ T$. It is easy to verify that by
the definition of automorphism $L(F)=L(F')$, so $F$ is also an SMML code-book,
and furthermore
\[
\tilde{x}_{\epsilon}(\theta_\epsilon)=\{U(x)|x\in\tilde{x}_\epsilon(\theta^*)\},
\]
so $F(x)=\theta_\epsilon\in\hat{\theta}_{\epsilon}(x)$.

Consider now a sequence of such $\theta_\epsilon$ for $\epsilon\to 0$.
The set $\Theta$ is a complete metric space, by construction the
$\theta_\epsilon$ reside inside the nested sets $\hat{\theta}_{\epsilon}(x)$,
and by our assumption that the problem is concentrated, for a small enough
$\epsilon$, $\hat{\theta}_{\epsilon}(x)$ is bounded. We conclude, therefore,
that the sequence $\theta_\epsilon$ has a converging sub-sequence. Let
$\theta$ be a bound for one such converging sub-sequence.

We claim that $\theta$ is inside both
$\hat{\theta}_{\text{SMML}}(x)$ and $\hat{\theta}_{\text{IP}}(x)$, thus
proving that their intersection is non-empty.

To show this, consider first that we know $\theta\in\hat{\theta}_{\text{IP}}(x)$
because $R$ is a continuous function, and by construction
$R_\theta(x)=R^*_{\theta}$.

Lastly, for every $\epsilon$ in the sub-sequence,
$\theta_\epsilon\in\hat{\theta}_{\text{SMML}}(x)$, so
$\theta\in\hat{\theta}_{\text{SMML}}(x)$ follows from the closure of the
SMML estimator (which is guaranteed by definition).
\end{proof}

\begin{cor}\label{C:inconsistent}
For the scale free Neyman-Scott problem with fixed $N$ and $J$,
$\widehat{\sigma^2}_{\text{SMML}}(s,m)=s^2$
and
$\widehat{\mu_n}_{\text{SMML}}(s,m)=m_n$.

In particular, this is true when $N$ approaches infinity, leading SMML to be
inconsistent for this problem.
\end{cor}

\begin{proof}
Recall that the ideal point is defined as the $x$ value
minimising $R_{(\sigma^2, \mu)}(x)$.

The general formula for $R$ in the scale-free Neyman-Scott problem was
derived in Section~\ref{SS:types} and is given in \eqref{Eq:R_NS}.
Differentiating $R$ according to $s^2$
and according to each $m_n$ we reach the following single-valued solution.
\begin{equation}\label{Eq:NSIP}
\widehat{\sigma^2}_{\textit{IP}}(x)=s^2,\quad
\widehat{\mu_n}_{\text{IP}}(x)=m_n.
\end{equation}
This is identical to the Maximum Likelihood estimate, and well known to be
inconsistent. The value
of the SMML estimator therefore follows from Theorem~\ref{T:SMML_IP}.

As a consistent estimator for $\sigma^2$ is $\frac{J}{J-1}s^2$ and not $s^2$,
the SMML estimator is inconsistent.
\end{proof}

At first glance, this result may seem impossible, because, as established,
an SMML code-book can only encode a countable number of $\theta$ values.
Corollary~\ref{C:uncountable} resolves this seeming paradox.

\begin{cor}\label{C:uncountable}
The scale free Neyman-Scott problem with fixed $N$ and $J$ admits
uncountably many distinct SMML code-books, and for every $(s,m)$
value there is a continuum of SMML estimates.
\end{cor}

\begin{proof}
Uncountably many distinct code-books can be generated by arbitrarily scaling
and translating any given code-book, which, as we have seen, does not alter
$L(F)$.

To show that for every $(s,m)$ value there are uncountably many distinct SMML
estimates, recall from our proof of Lemma~\ref{L:NSconcentrated}
that if we consider the problem in $(\log s,m/s)$
observation space and $(\log \sigma,\mu/\sigma)$ parameter space, then both
scaling and translation in the original parameter space are translations under
the new representation. If any $x$ belongs to a region of volume $V$ in this
space that is mapped to a particular $\theta$ by a particular $F$, one can
create a new code-book, $F'$, which is a translation of $F$ in both
$(\log s,m/s)$ and $(\log \sigma,\mu/\sigma)$, which would still be optimal.

As long as the translation in observation-space is such that $x$ is still
mapped into its original region, its associated $\theta'$ will be the
correspondingly-translated $\theta$. As such, the volume of $\theta$ values
associated with a single $x$ is at least as large the volume of the region
of $x$ (and, by observation-transitivity of the problem, at least as large as the volume
of the largest region in the code-book's partition).
\end{proof}

SMML is therefore not a point estimator for this problem at all.

\subsection{Relating IP to ML}

Beyond the connections between the SMML solution and the Ideal Point
approximation, there is also a direct link to the Maximum Likelihood estimate.

\begin{thm}\label{T:IP_ML}
If $(\mathbf{x},\boldsymbol{\theta})$ is a homogeneous estimation
problem, then $\hat{\theta}_{\textit{IP}}=\hat{\theta}_{\textit{ML}}$.
\end{thm}

\begin{proof}
By definition,
\[
\tilde{x}_{\text{IP}}(\theta)=\argmin_{x\in X} R_{\theta}(x)
=\{x\in X | R_{\theta}(x)=\min_{x'\in X} R_{\theta}(x')\}.
\]

By assumption, the estimation problem is parameter-homogeneous, so
$\min_{x\in X} R_{\theta}(x)$
is a constant, $R^*$, independent of $\theta$.
Substituting $R^*$ into the definition of $\tilde{x}_{\text{IP}}$ and
calculating the functional inverse, we get
\[
\hat{\theta}_{\text{IP}}(x)=\{\theta\in\Theta | R_{\theta}(x)=R^*\}.
\]

For an arbitrary choice of $\theta_0$,
let $x_0$ be such that $x_0\in \tilde{x}_{\text{IP}}(\theta_0)$.
The value of $R_{\theta_0}(x_0)$ is $R^*$, and there
certainly is no $\theta'\in\Theta$ for which
$R_{\theta'}(x_0)<R^*$ (or this would contradict parameter-homogeneity),
so, using the notation of Definition~\ref{D:homogeneous},
$R_{\textit{opt}}=R_{\textit{opt}}(x_0)=R^*$.

Thus,
\begin{align*}
\hat{\theta}_{\text{IP}}(x)&=\{\theta\in\Theta | R_{\theta}(x)=R^*\} \\
&=\argmin_{\theta\in\Theta} R_{\theta}(x) \\
&=\argmin_{\theta\in\Theta} \log\left(\frac{r(x)}{f(x|\theta)}\right) \\
&=\argmax_{\theta\in\Theta} f(x|\theta)
=\hat{\theta}_{\text{ML}}(x).
\end{align*}
\end{proof}

\begin{cor}\label{C:SMML_ML}
In any regular estimation problem, for every $x\in X$,
\[
\hat{\theta}_{\text{SMML}}(x) \cap \hat{\theta}_{\text{ML}}(x) \ne \emptyset.
\]
\end{cor}

\begin{proof}
From Lemma~\ref{L:comprehensive} we know every observation-transitive problem
is observation-homogeneous, so we can apply both Theorem~\ref{T:SMML_IP}, equating
the SMML estimator with the IP one, and Theorem~\ref{T:IP_ML}, equating the
IP one with ML.
\end{proof}

\appendices

\section{Analysis of MML approximations}\label{S:approximations}

We show regarding the MML approximations used by
\citet{dowe1997resolving} and \citet{Wallace2005}, respectively, that on
the scale-free Neyman-Scott problem, i.e.\ the Neyman-Scott problem under
a scale-free prior, both converge to the ML estimate, and are therefore
also, like SMML, not consistent for the problem.

While not satisfying consistency, the fact that all three algorithms converge
to the same limit does suggest that, at least for this problem, the two
MML algorithms studied are good approximations, adequately modelling the
limit behaviour of SMML.

\subsection{The Wallace-Freeman approximation}\label{SS:WF}

Perhaps the most widely used variant of MML is the Wallace-Freeman
approximation.
\begin{defi}
\emph{The Wallace-Freeman estimator (WF-MML)} is
\begin{equation}\label{Eq:wf}
\hat{\theta}_{\text{WF}}(x)\defeq\argmax_{\theta}\frac{f(\theta|x)}{\sqrt{|\mathcal{I}(\theta)|}}
=\argmax_{\theta} f(x|\theta) \frac{h(\theta)}{\sqrt{|\mathcal{I}(\theta)|}},
\end{equation}
where $\mathcal{I}(\theta)$ indicates the
Fisher information \citep{lehmann2006theory, savage1976rereading}.
\end{defi}

This was derived in \citet{WallaceFreeman1987} by use of a quadratic
approximation to the message length.

Proving that WF-MML is not consistent for the scale-free Neyman-Scott problem
is a direct corollary of the following (straightforward) theorem.
We list it as folkloric because, although we could not
find it proved explicitly in the literature, it is clearly a known result.
For example, it is alluded to in \cite[p. 412]{Wallace2005}.

It details the behaviour of WF-MML on problems that have a
\emph{Jeffreys prior} \citep{Jeffreys1946invariant, Jeffreys1998theory}.
A Jeffreys prior is a prior satisfying for all
$\theta\in\Theta$,
\[
h_{\text{Jeffreys}}(\theta) \propto \sqrt{|\mathcal{I}(\theta)|}.
\]
It is one of the canonical non-informative priors computable for any
frequentist estimation problem, and is routinely used, e.g., by
objective Bayesians, in lieu of other information.

\begin{thm}[folklore]\label{T:mml}
In any estimation problem $(\mathbf{x},\boldsymbol{\theta})$, where
$\boldsymbol{\theta}$ is distributed according to the problem's Jeffreys prior,
the estimate of the Wallace-Freeman approximation coincides with ML.
\end{thm}

\begin{proof}
By definition, the Jeffreys prior is proportional to the square root of the
determinant of the Fisher information matrix.
Hence, when substituting this prior for $h(\theta)$ in \eqref{Eq:wf}, the
only part of the expression that remains dependent on $\theta$ is
$f(x|\theta)$, the likelihood. Thus, it becomes a Maximum Likelihood estimator.
\end{proof}

An immediate corollary is therefore
\begin{cor}\label{C:bad_prior}
For any frequentist estimation problem over which ML is inconsistent, there
exists a (possibly improper) prior such that the Wallace-Freeman approximation
is also inconsistent over the same problem under said prior.
\end{cor}

\begin{proof}
The Jeffreys prior is an example of such a prior, because, by
Theorem~\ref{T:mml}, under this prior the Wallace-Freeman approximation
coincides with ML.
\end{proof}

\begin{cor}
WF-MML is not consistent over the Neyman-Scott problem with a scale-free
prior. Its asymptotic behaviour for this problem is identical to that of ML.
\end{cor}

\begin{proof}
This follows immediately from the previous results, because the
scale-free prior is a Jeffreys prior for this problem, as can be ascertained
directly by computing the Fisher information matrix.
\end{proof}

\subsection{Ideal Group}\label{SS:idealgroup}

In this section, we analyse the Ideal Group MML approximation, and show that
it is inconsistent for the Neyman-Scott problem under the scale-free prior,
by utilising our notion of an ideal point, defined in Section~\ref{SS:ip}.

\begin{thm}\label{T:IG}
The Ideal Group estimator is not consistent for the scale-free
Neyman-Scott problem. In particular, it contains for $(\sigma,\mu)$
the point $(s,m)$, which is the (inconsistent) Maximum Likelihood estimate,
as the Ideal Point.
\end{thm}

\begin{proof}
Recall that the ideal point is defined as the $x$ value
minimising $R_{\theta}(x)$, and for this reason guarantees that
the ideal group for $\theta$ necessarily includes it.

The Ideal Point estimate for the scale-free Neyman-Scott problem was given
in \eqref{Eq:NSIP}, in the proof of Corollary~\ref{C:inconsistent} in
Section~\ref{SS:SMML_IP}.
This estimate is identical to the Maximum Likelihood estimate,
\[
\hat{\theta}_{\text{ML}}(\sigma,\mu)=(s,m),
\]
and is well known to be inconsistent, for which reason the Ideal Group solution
is also inconsistent.
\end{proof}

\section{Proof that Neyman-Scott is local}\label{S:NSlocal}

We prove Lemma~\ref{L:NSlocal}, stating that the
scale free Neyman-Scott problem is local.

\begin{proof}
Set $k=2N+1+c^N$, for a $c$ value to be chosen later on.
Importantly, $k$, $c$ and all other constants introduced later on in this
proof (e.g., $T$, $\Delta$ and $\Delta'$) depend solely on $N$ and $J$ and
are not dependent on $\theta$. As such, they are constants of the
construction.

Let $T=N \log(c+1)$, and for $n=1,\ldots,N$
let $\mu^{n+}$ be the vector identical to $\mu$ except that its $n$'th
element equals $\mu_n+\sigma\sqrt{2 T}$.
Let $\mu^{n-}$ be the vector identical to $\mu$
except that its $n$'th element equals $\mu_n-\sigma\sqrt{2 T}$.

For $\theta_1,\ldots,\theta_{2N}$, we use all
$(\sigma, \mu^{n+})$ and all $(\sigma, \mu^{n-})$.
Next, we pick $\theta_{2N+1}=(e\sigma,\mu)$, where $e$ is
Euler's constant.

This leaves a further $c^N$ values of $\{\theta_i\}$ to be assigned. To assign these,
divide for each $n=1,\ldots,N$ the range between $\mu^{n-}$ and $\mu^{n+}$
into $c$ equal-length segments, and let $\Omega_n$ be the set containing the
centres of these segments. We define our remaining $\theta$ values as
\[
\Theta'=\left\{\left(\frac{\sqrt{2NJ}}{c}\sigma,\mu_1',\ldots,\mu_N'\right)\middle|\forall n, \mu_n'\in\Omega_n\right\}.
\]

We will show that, for a constant $V_0$ to be chosen later on,
outside a subset of $X$ of total scaled probability $V_0$,
\[
e^T f(x|\theta) < \max_i f(x|\theta_i).
\]
Equivalently:
\begin{equation}\label{Eq:logT}
\max_i \log f(x|\theta_i) - \log f(x|\theta) > T.
\end{equation}

Showing this is enough to prove the lemma, because for a sufficiently large
$c$,
\[
e^T=(c+1)^N\ge c^N+Nc^{N-1}>c^N+2N+2=k+1,
\]
so by choosing $\gamma=\frac{k+1}{k}$ the conditions of
Definition~\ref{D:local} are satisfied. (Recall that $k$ is a constant of the
construction, and therefore $\gamma$ can depend on $k$.)

To prove \eqref{Eq:logT}, let us divide the problem into cases. First, let us
show that this holds true for any $x=(s,m)$ value for which, for any $n$,
$|m_n-\mu_n|>\sigma\sqrt{2T}$. To show this, assume without loss of generality
that for a particular $n$ the equation $m_n-\mu_n>\sigma\sqrt{2T}$ holds true.
\begin{align*}
\log \max_i f(x|\theta_i) & -\log f(x|\theta) \\
&\ge \log f(s,m|\sigma,\mu^{n+}) -\log f(s,m|\sigma,\mu) \\
&=-JT+\frac{J\sqrt{2T}}{\sigma}(m_n-\mu_n) \\
&>-JT+2JT
>T.
\end{align*}

Next, we claim that there is a $\Delta$ value such that if
$s/\sigma>\Delta$, \eqref{Eq:logT}
holds. This can be demonstrated as follows.
\begin{align*}
\log & \max_i f(x|\theta_i) -\log f(x|\theta) \\
& \ge \log f(s,m|e\sigma,\mu) -\log f(s,m|\sigma,\mu) \\
&=-NJ+\left(1-\frac{1}{e^2}\right)\frac{NJs^2+J\sum_{n=1}^{N} (m_n-\mu_n)^2}{2\sigma^2} \\
&> -NJ + \left(1-\frac{1}{e^2}\right)\frac{NJ}{2}\Delta^2.
\end{align*}
By choosing a high enough value of $\Delta$, this lower bound can be made
arbitrarily large. In particular, it can be made larger than $T$, making
\eqref{Eq:logT} hold.

Our last case is one where $s/\sigma<\Delta'$, for some $\Delta'$ to be
computed. In considering this case, we can assume that for every $n$,
$|m_n-\mu_n|\le\sigma\sqrt{2T}$, or else \eqref{Eq:logT} holds due to our
first claim. With this assumption, the value of $|m_n-\mu_n'|$ for every $n$
is at most $\sigma\sqrt{2T}/c$ for some $\mu_n'\in \Omega_n$. Let
$(\frac{\sqrt{2NJ}}{c}\sigma,\mu')$ be the element of $\Theta'$ with $\mu'$
closest to $m$ in every coordinate.
\begin{align}
\log & \max_i f(x|\theta_i)-\log f(x|\theta) \nonumber\\
&\ge \log f\left(s,m\middle|\frac{\sqrt{2NJ}}{c}\sigma,\mu'\right)-\log f(s,m|\sigma,\mu) \nonumber\\
&\ge\log\left[\frac{c^{NJ}}{(2 \sqrt{\pi N J}\sigma)^{NJ}} e^{-\frac{c^2}{2 N J}\frac{NJs^2+2NJT\sigma^2/c^2}{2\sigma^2}}\right] \nonumber\\
&\quad\quad -\log\left[\frac{1}{(\sqrt{2\pi}\sigma)^{NJ}} e^{-\frac{NJs^2}{2\sigma^2}}\right] \nonumber\\
&=NJ\log\left(\frac{c}{\sqrt{2 N J}}\right) + \left(1-\frac{c^2}{2 N J}\right) \frac{NJs^2}{2\sigma^2} -\frac{T}{2}.\label{Eq:difflhs}
\end{align}

Because for a large enough value of $c$, the expression $1-c^2/2 N J$ is
negative, the value of \eqref{Eq:difflhs} is minimised when $s/\sigma$ is
maximal. Therefore,
\[
\left(1-\frac{c^2}{2 N J}\right) \frac{NJs^2}{2\sigma^2}
> \left(1-\frac{c^2}{2 N J}\right) \frac{NJ\Delta'^2}{2},
\]
which, together with \eqref{Eq:difflhs}, leads to
\begin{align*}
\log & \max_i f(x|\theta_i) -\log f(x|\theta) \\
&> NJ\log\left(\frac{c}{\sqrt{2 N J}}\right) + \left(1-\frac{c^2}{2 N J}\right) \frac{NJ\Delta'^2}{2} -\frac{T}{2}.
\end{align*}

The value of $\Delta'$ can be made arbitrarily small. For example, we
may set $\Delta'$ to satisfy
\[
\left(\frac{c^2}{2 N J}-1\right) \frac{NJ\Delta'^2}{2} \le \frac{T}{4}.
\]
If we set $\Delta'$ in this way, it only remains to be proved that
\[
NJ\log\left(\frac{c}{\sqrt{2 N J}}\right)-\frac{T}{4}-\frac{T}{2}\ge T.
\]
Substituting in the definition of $T$ and simplifying, we get
\[
\frac{1}{(2 N J)^{2J}} c^{4J} \ge (c+1)^7.
\]
Considering that $J\ge 2> 7/4$, the left-hand side is a polynomial of higher
degree than the right-hand side. Therefore, for a large-enough $c$, the
equation holds.

We have therefore shown that \eqref{Eq:logT} holds for every $x\in X$, except
within a bounding box of size
$V=\left(2\sqrt{2T}\right)^N \log \frac{\Delta}{\Delta'}$
in $(\log s,m/\sigma)$-space, a size that is independent of $\theta$.
Because this bounding box bounds $s/\sigma$ from
below by a constant $\Delta'$, its volume is also bounded by
$V_0=V/\Delta'^N$ in $(\log s,m/s)$-space,
and this value is also independent of $\theta$.

Recall, however, that volume in $(\log s,m/s)$-space equals
(or is proportional to)
scaled probability in $X$. Equation \eqref{Eq:logT} holds, therefore,
everywhere except in a subset whose scaled probability is bounded by a constant
independent of $\theta$.
\end{proof}

\section*{Acknowledgement}
The author would like to thank Graham Farr for his extensive comments
on multiple drafts of the paper.

Thank you also to David L.\ Dowe for his considerable help in
preparing this manuscript, and especially for pointing out that the
scale free prior we use coincides with the problem's
Jeffreys prior.

Additionally, thank you to this paper's editor, Ioannis Kontoyiannis, and
to two anonymous reviewers for their substantial suggestions which
considerably improved the paper's presentation.

\bibliographystyle{IEEEtran}
\bibliography{bibneymanscott}

\begin{IEEEbiographynophoto}{Michael Brand}
received a Ph.D. in Information Technology from Monash University, Australia,
in 2013, an M.Sc. in Mathematics and Computer Science (specialising in
Applied Mathematics) from The Weizmann Institute of Science,
Israel, in 2006, and a B.Sc. in Industrial Engineering from
Tel-Aviv University, Israel, in 1995. He is an Adjunct A/Prof in Data Science
and Artificial Intelligence at Monash University, where he also headed the
Monash Centre for Data Science, and is the head and founder
of analytics consultancy firm Otzma Analytics. He previously served as
Chief Data Scientist at Telstra Corporation, Chief Scientist at Verint
Systems, Senior Principal Data Scientist at Pivotal and CTO Group
Algorithm Leader at PrimeSense.
His research interests include theoretical machine learning,
the theory of computation, signal modelling, Big Data computing, and
combinatorics, and he holds 16 patents in computer vision,
natural language processing and database algorithms.
\end{IEEEbiographynophoto}

\end{document}